\def\year{2022}\relax
\documentclass[letterpaper]{article} 
\usepackage{aaai22}  
\usepackage{times}  
\usepackage{helvet}  
\usepackage{courier}  
\usepackage[hyphens]{url}  
\usepackage{graphicx} 
\usepackage{natbib}  
\usepackage{caption} 
\DeclareCaptionStyle{ruled}{labelfont=normalfont,labelsep=colon,strut=off} 
\usepackage{algorithm}
\usepackage{algorithmic}
\usepackage{color}

\usepackage{amsthm,amsmath,amssymb}
\usepackage{mathrsfs}
\newtheorem{theorem}{Theorem}
\newtheorem{definition}{Definition}
\usepackage{subfigure}
\usepackage{multirow}
\usepackage{booktabs}

\usepackage{newfloat}
\usepackage{listings}
\floatstyle{ruled}
\newfloat{listing}{tb}{lst}{}
\floatname{listing}{Listing}
\title{Quantifying Robustness to Adversarial Word Substitutions}
\author{
    Yuting Yang,\textsuperscript{\rm 1,2}\equalcontrib \ 
    Pei Huang,\textsuperscript{\rm 2,3}\equalcontrib \
    FeiFei Ma,\textsuperscript{\rm 2,3,4}\footnote{Corresponding Authors} \
    Juan Cao\textsuperscript{\rm 1,2} \
    Meishan Zhang \textsuperscript{\rm 5}\
    Jian Zhang\textsuperscript{\rm 2,3}\footnotemark[2]
    Jintao Li \textsuperscript{\rm 1}
}
\affiliations{
    \textsuperscript{\rm 1} Key Lab of Intelligent Information Processing,\\
    Institute of Computing Technology, Chinese Academy of Sciences, Beijing, 100190, China\\
    \textsuperscript{\rm 2} University of Chinese Academy of Sciences, Beijing, 100049, China\\
    \textsuperscript{\rm 3}State Key Laboratory of Computer Science,\\
    Institute of Software, Chinese Academy of Sciences (ISCAS), Beijing, 100190, China\\
    \textsuperscript{\rm 4} Laboratory of Parallel Software and Computational Science, ISCAS, Beijing, 100190, China\\
    \textsuperscript{\rm 5} Institute of Computing and Intelligence, Harbin Institute of Technology (Shenzhen), China
    yangyuting@ict.ac.cn, \{huangpei, maff\}@ios.ac.cn, caojuan@ict.ac.cn, \\
    mason.zms@gmail.com, zj@ios.ac.cn, jtli@ict.ac.cn
}

\usepackage{bibentry}

\begin{document}

\maketitle

\begin{abstract}
    Deep-learning-based NLP models are found to be vulnerable to word substitution perturbations. Before they are widely adopted, the fundamental issues of robustness need to be addressed. Along this line, we propose a formal framework to evaluate word-level robustness. First, to study safe regions for a model, we introduce robustness radius which is the boundary where the model can resist any perturbation. As calculating the maximum robustness radius is computationally hard, we estimate its upper and lower bound. We repurpose attack methods as ways of seeking upper bound and design a pseudo-dynamic programming algorithm for a tighter upper bound. Then verification method is utilized for a lower bound. Further, for evaluating the robustness of regions outside a safe radius, we reexamine robustness from another view: quantification. A robustness metric with a rigorous statistical guarantee is introduced to measure the quantification of adversarial examples, which indicates the model's susceptibility to perturbations outside the safe radius. The metric helps us figure out why state-of-the-art models like BERT can be easily fooled by a few word substitutions, but generalize well in the presence of real-world noises.

\end{abstract}

\section{Introduction}

Deep learning models have achieved impressive improvements on various NLP tasks. However, they are found to be vulnerable to input perturbations, such as paraphrasing \cite{SinghGR18}, inserting character \cite{BelinkovB18} and replacing words with similar ones \cite{Ren19}. In this paper, we focus on word substitution perturbation \cite{textfooler,NeekharaHDK19,seme} as shown in Figure \ref{introa}, in which the output of a model can be altered by replacing some words in the input sentence while maintaining the semantics. Before deep learning models are widely adopted in practice, understanding their robustness to word substitution is critical. 

\begin{figure}[ht]
	\centering
	\includegraphics[width=1.0\linewidth]{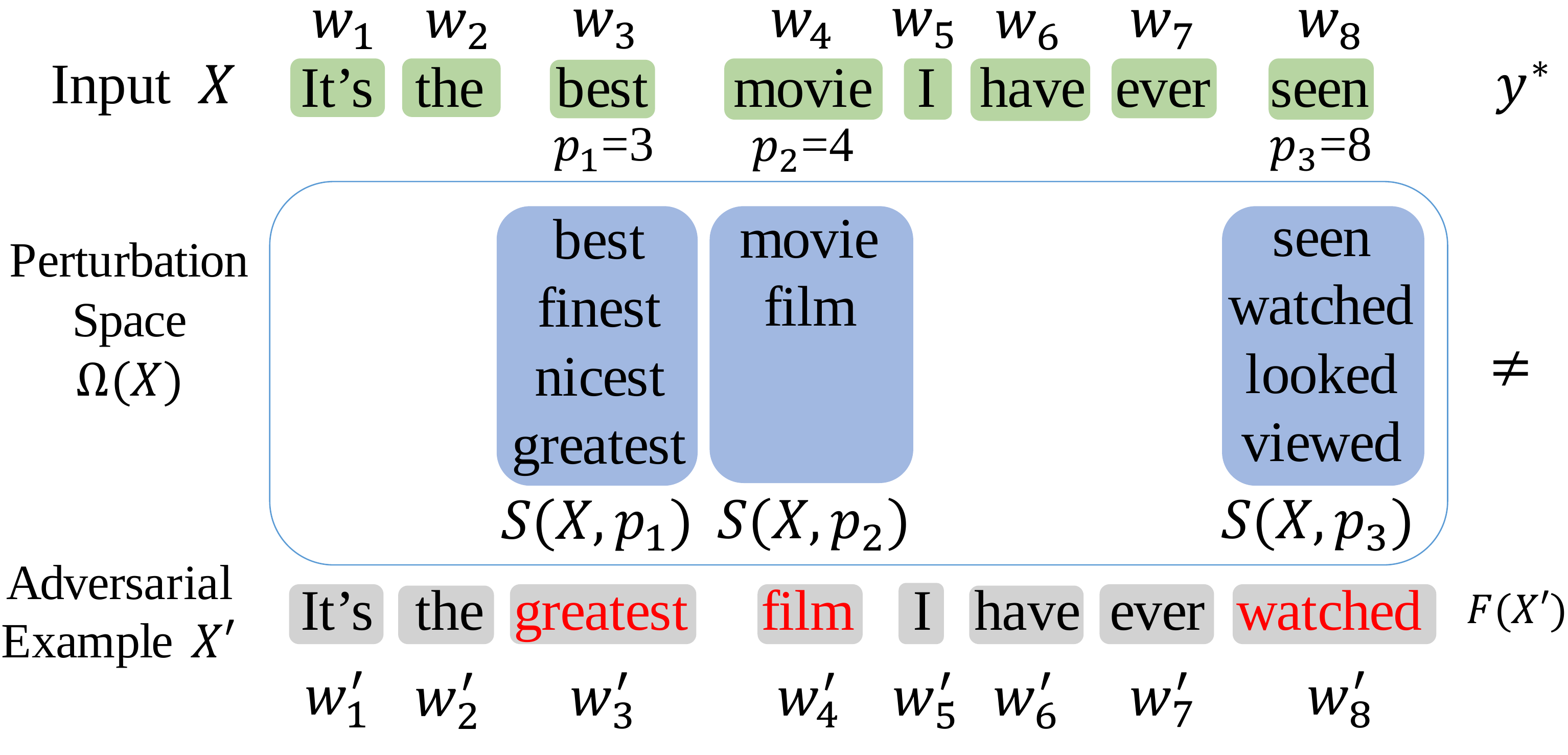}
	\caption{An example of word substitution perturbation.}
	\label{introa}
\end{figure}

\begin{figure}[ht]
	\centering
	\includegraphics[scale=0.27]{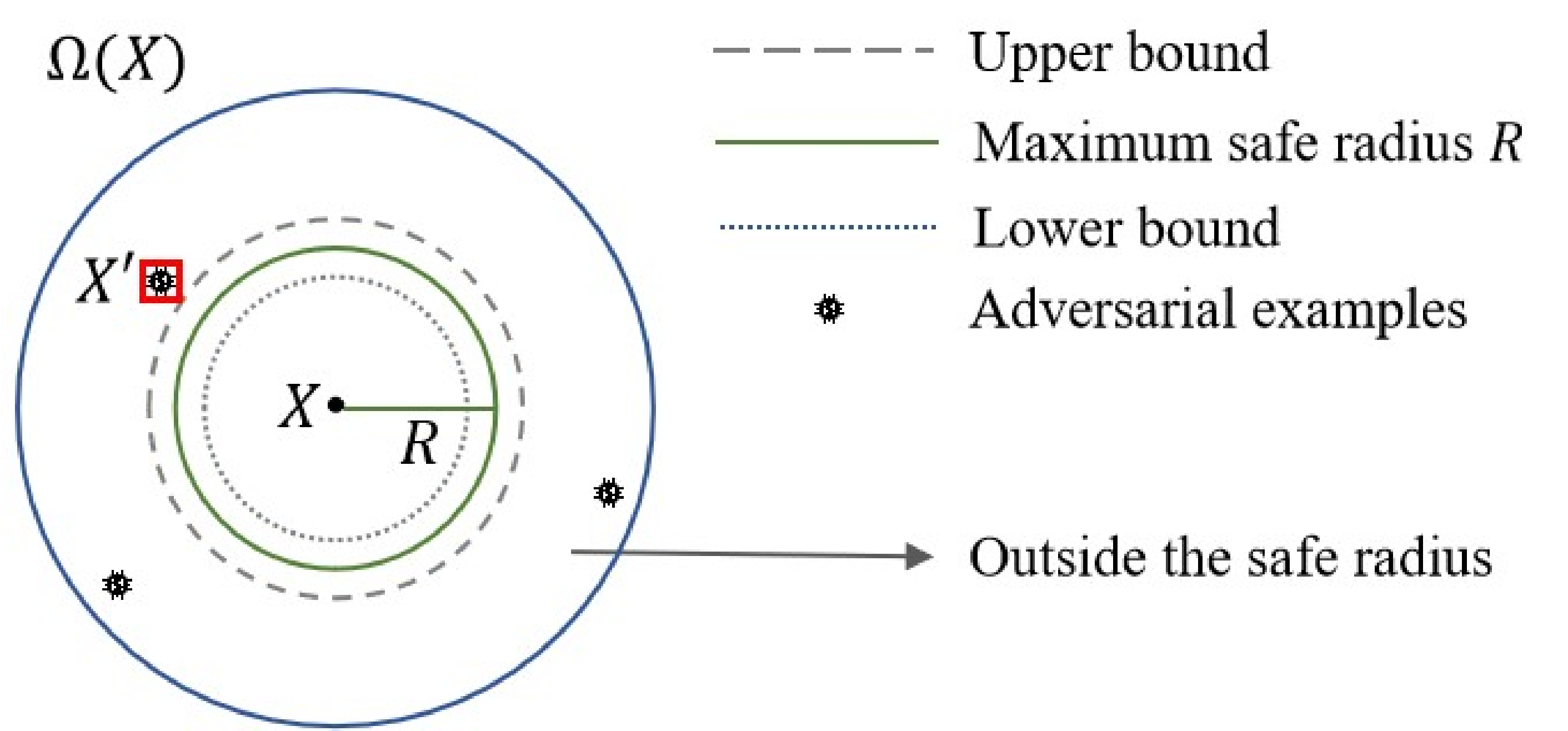}
	\caption{Diagram of our methods for evaluating robustness: (1) evaluating maximum safe radius $R$ via obtaining its upper and lower bound; (2) measuring the model's vulnerability outside the safe radius via a robustness metric $PR$.}
	\label{introb}
\end{figure}

In recent years, several studies focus on generating adversarial examples \cite{textfooler,Wang0D021} or certify the absence of adversarial examples in the whole perturbation space \cite{JiaRGL19,Huang19,YeGL20}. However, almost all current deep learning models are unable to be regarded as absolutely robust under such a yes-or-no binary judgment. Along this line, some deeper questions can be asked. Where is the safe boundary of a model to resist perturbation? Why can a well-trained NLP model be fooled by small perturbations but generalize well to real-world inputs with noises? Does the existence of an adversarial example in the exponential input space completely destroy the defense capability of the model?

To answer these questions more comprehensively, we propose a formal framework for evaluating models' robustness to word substitution from the view of quantification. We quantify the magnitude of the perturbation (or the number of substitutions) a model can resist. Figure \ref{introb} visualizes the problems we study in this paper. Robustness radius (safe radius) $r$, which is defined as the magnitude of the perturbation space where no adversarial examples exist, is useful for studying the safe regions of models. In particular, the maximum robustness radius $R$ depicts the boundary of perturbations a model can resist. Apart from safe regions, the vulnerability outside safe regions also needs to be evaluated as it can influence the model's performance in practice. A natural idea is to quantify the number of adversarial examples for a given radius as a metric for robustness.

The main challenge of the evaluation framework is that the perturbation space can be exponentially large, so solving these problems exactly is not feasible in many cases. To overcome this problem, we retreat from the exact computation of $R$ to estimate its upper and lower bounds. An adversarial example with fewer substitutions can provide a tighter upper bound for $R$. Therefore, we repurpose attack methods for evaluating upper bound and design an algorithm called pseudo-dynamic programming (PDP) to craft adversarial examples with as few substitutions as possible. Then, for the lower bound, we find that certifying word-level robustness with a fixed radius can be solved in polynomial time. So we use verification methods to give a lower bound. Finally, we introduce a robustness metric $PR$ which denotes the number of adversarial examples for a given radius. It can provide a quantitative indicator for the models' robustness outside the absolute safe radius. As it is a more difficult problem than calculating the maximum safe radius, we estimate the value of $PR$ with a rigorous statistical guarantee.

We design experiments on two important NLP tasks (text classification and textual entailment) and two models (BiLSTM and BERT) to study our methods empirically. Experiments show that PDP algorithm has a stronger search capability to provide a tighter upper bound for the maximum robustness radius. The robustness metric $PR$ results present an interesting phenomenon: although most well-trained models can be attacked by a few word substitutions with a high success rate, the word-substitution-based adversarial examples distribute widely in perturbation space but just occupy a small proportion. For example, BERT can be successfully ($>89.7\%$) attacked by manipulating $4.5$ words on average on IMDB. However, more than $90.66\%$ regions can resist random word perturbations with a high probability ($>0.9$). We conclude that some adversarial examples may be essentially on-manifold generalization errors, which can explain the reason why these ``vulnerable'' models can generalize well in practice.

\section{Preliminary}
Given a natural language classifier $F:\mathcal{X}\rightarrow \mathcal{Y}$, which is a mapping from an input space to an output label space. The input space $\mathcal{X}$ contains all possible texts $X=w_1, w_2,..., w_n$ and output space $\mathcal{Y}=\{y_1,y_2,...,y_c\}$ contains $c$ possible predictions of an input. $w_i$ is usually a word embedding or one-hot vector. $F_y(\cdot)$ is the prediction score for the $y$ label. Let $P=\{p_1,p_2,..., p_m\}$ be the set of perturbable positions. For each perturbable position $p \in P$, there is a set $S(X,p)$ which contains all candidate words for substitution without changing the semantics (the original word $w_{p}$ is also in $S(X,p)$). Figure \ref{introb} is a schematic diagram and contains explanations of some notations.

\begin{definition} [Adversarial Example]
Consider a classifier $F(x)$. Given a sequence $X$ with gold label $y^*$ and $X'=w_1', w_2',..., w_n'$ which is a text generated by perturbing $X$, $X'$ is said to be an adversarial example if:
\begin{equation}
F(X') \neq y^* \label{Eq:AdvExample}
\end{equation}
\end{definition}

\begin{definition}
	A perturbation space $\Omega(X)$ of an input sequence $X$ is a set containing all perturbations generated by substituting the original word by candidate words in $S(X,p)$ for each perturbable position $p \in \mathcal{P}$. 
\end{definition}
The cardinality of $\Omega(X)$ is $\prod_{p \in \mathcal{P}} |S(X,p)|$.

\begin{definition}[Word-level Robustness]
	Consider a classifier $F(x)$. Given a sequence $X$ with gold label $y^*$, classifier $F$ is said to be robust in the perturbation space $\Omega(X)$ if the following formula holds:
	\begin{equation} \label{Eq:Robustness}
	\forall X'.\ X' \in \Omega(X) \Rightarrow F(X')=y^*  
	\end{equation}
\end{definition}

If a classifier is not robust in $\Omega(X)$, we also want to know what maximum perturbation it can resist. We use $L_0$ distance $r$ to describe the degree of perturbation, which is also called \textbf{robustness radius} or \textbf{safe radius}. The maximum robustness radius is denoted as \textbf{$R$}. 

\begin{definition}[Word-level $L_0$-Robustness]\label{eq:L0}
	Consider a classifier $F(x)$. Given an $L_0$ distance $r$ and a sequence $X$ with gold label $y^*$. Let $\Omega_{r}(X):=\{X' : X' \in \Omega(X) \wedge \left\|X'-X\right\|_0\leq r\}$ and $\left\|\cdot \right\|_0$ denote the number of substituted words. The classifier $F$ is said to be robust with respect to $\Omega_{r}(X)$ if the following formula holds true:
	\begin{equation} \label{Eq:L0Robustness}
	\forall X'.\ X' \in \Omega_r(X) \Rightarrow F(X')=y^* 
	\end{equation}
\end{definition}

If formula (\ref{Eq:L0Robustness}) is true, that means neural network can resist any substitutions in $\Omega_r(X)$. For the point-wise robustness metric, substitution length and ratio can be easily converted to each other.


\subsection{Problems}
	From a high-level perspective, there are four types of relevant problems:
	\begin{itemize}
		\item \textbf{Type-1} (\texttt{Satisfaction problem}). Find an adversarial example in the perturbation space. It helps to prove that a neural network is unsafe in a certain input space. 
		\item \textbf{Type-2} (\texttt{Optimization problem}). Find the adversarial example with minimal perturbation. This can help us figure out the boundary of safe regions.
		\item \textbf{Type-3} (\texttt{Proving problem}). Certify the absence of adversarial examples in the perturbation space. In other word, prove that the formulas like (\ref{Eq:Robustness}) or (\ref{Eq:L0Robustness}) is true. This problem can prove that the network is absolutely safe in certain input spaces.
		\item \textbf{Type-4} (\texttt{Counting problem}). Give the number of adversarial examples in the perturbation space. It further investigates the model's susceptibility outside the absolutely safe radius.
	\end{itemize}
	
    In recent years, most relevant works focused on developing effective attacking algorithms for generating adversarial examples \cite{JiaRGL19,Huang19,YeGL20}, which can be viewed as ``\textbf{Type-1 problem}": finding an adversarial example in the perturbation space. However, finding an adversarial example or not can not reflect the model's defense ability in the whole perturbation space. Type-2$\sim$4 problems are more informative and remains to be studied, which are our focuses in this work.
	
	These four problems present different levels of difficulty. In more details, Type-1 problem is in $\mathcal{NP}$; Type-2 problem is $\mathcal{NP}$-hard; Type-3 problem is in $\mathcal{C}o\mathcal{NP}$ \cite{KatzBDJK17}; and Type-4 problem is $\#\mathcal{P}$-hard ($\mathcal{NP}\subseteq \#\mathcal{P}$) \cite{BalutaSSMS19}. Type-3 problem is the complement of Type-1 problem. Sometimes, Type-1 and Type-3 problems are not strictly distinguished, so certifying robustness is sometimes said to be in $\mathcal{NP}$ as well. These conclusions are drawn when the tasks and neural networks have no restrictions.

\section{Methods}
In this section, we propose a formal framework for evaluating robustness to word substitution perturbation. We first study the upper and the lower bound for the safe boundary which belongs to Type 2 and Type 3 problem respectively. Then we use a statistical inference method to quantify the adversarial examples outside a safe region with a rigorous guarantee, which is a Type 4 problem. For a more clear state, we organize the following sections according to the three problems.


\subsection{Type-2 Problem: Pseudo-Dynamic Programming for Crafting Adversarial Examples}\label{PDPCAE}
 As shown in Figure \ref{introa}, if an adversarial example is found in $\Omega_r(X)$, it means $\Omega_r(X)$ is not $L_0$-Robust according to Definition 4 or the maximum robustness radius of the model must be lower than $r$. An adversarial example offers an upper bound for the maximum robustness radius. Naturally, we wonder about a tighter upper bound for estimating safe boundaries. So, we design an efficient algorithm to find adversarial examples with fewer substituted words in $\Omega(X)$. The algorithm can not only find high-quality adversarial examples, but also provide a tighter upper bound for robustness radius in $\Omega_{r}(X)$. The basic idea of our method is inspired by dynamic programming. 
\subsubsection{Methodology}
Finding the optimal adversarial example $X'$ can be seen as a combinatorial optimization problem with two goals:
\begin{itemize}
	\item [i)] Optimize the output confidence score of $F(X')$ to fool the classifier.
	\item [ii)] Minimize the number of substituted words (i.e. Minimize $\left\|X'-X\right\|_0$).
\end{itemize}

We consider the optimizing procedure is in correlation with a time variable $t$. Let $A(X, t)$ denote the text set containing all combinations of word substitutions for first $t$ perturbed positions $\{p_1, p_2, ...p_t\}$. $\mathrm{Opt}[(A(X, t))]$ denotes the operation to get the optimal adversarial example from $A(X, t)$. Operation $A(X,t-1) \times S(X,p_t)$ means substitute the $p_t$-th position with candidate words in $S(X,p_t)$ for all texts in $A(X,t-1)$. Then we get the optimal adversarial example from $A(X,m)$ in $m$ steps:
\begin{equation*}
\mathrm{Opt}[A(X, t)]:=\mathrm{Opt}[A(X,t-1) \times S(X,p_t)]
\end{equation*}
where $|P|$=$m$, $t \in \{1,..., m\}$ and $A(X,0)$=$\{X\}$. This procedure can guarantee to find the optimal adversarial example. However, it has exponential time complexity as the size of $A(X, t)$ increases exponentially with $t$. 

We make some relaxations for this procedure to ensure it can be executed in polynomial time. At step $t$, we only keep top $K$ texts in $A(X,t-1)$ which are considered to be more promising in generating adversarial examples. The others will be forgotten at this step. In this context, we have:
\begin{equation}\label{PDP}
A(X, t):=\mathrm{TopK}(A(X,t-1)) \times S(X,p_t)
\end{equation}

This relaxation comes at the cost of the guarantee of finding the optimal adversarial example. Due to that, the recurrence relation \ref{PDP} is similar to the dynamic programming equation, we call it \textbf{pseudo-dynamic programming (PDP)}.

\begin{algorithm}[tb]
	\caption{PDP (Pseudo-Dynamic Programming)}
	\label{alg:algorithm1}
	\begin{algorithmic}[1] 
		\REQUIRE ~\\
		 $F$: A classifier \\
		 $X$: An input text with $n$ words \\
		 $\tau$: the maximum percentage of words for modification.
		 \ENSURE An adversarial example $X'$ \textbf{\textit{or}} Failed.
		\STATE $A(X,0) \leftarrow \{X\}$ ;
		\STATE $P_1\leftarrow \emptyset, P_2\leftarrow \{p_1,p_2,...p_m\}$	;
		\FORALL {$t \leftarrow 1$ \TO $m$}
		\STATE  $A(X,t-1) \leftarrow \mathrm{TopK}(A(X,t-1))$;
		\STATE $p* \leftarrow \mathop{\arg\max}_{p \in P_2}\{I_p(A(X,t-1),p)\}$;
		\STATE $A(X, t) \leftarrow A(X,t-1) \times S(X,p^*)$;
		\STATE $P_1\leftarrow P_1\cup \{p*\}, P_2\leftarrow P_2\setminus \{p*\}$;
		\IF{$\exists X' \in A(X, t), F(X') \neq y*$ \textbf{and} $\left\|X'-X\right\|_0 \leq \tau \cdot n$}
		\STATE $X' \leftarrow$ the best adversarial example in $A(X, t)$
		\RETURN $X'$ 
		\ENDIF
		\ENDFOR
		\RETURN Failed
	\end{algorithmic}
\end{algorithm}

Notice that the number of substituted words of all texts in  $A(X,t)$ is less than $t$. So, when an adversarial example is found at an earlier time $t$, it has greater chances to achieve the goal (ii) better. So, we make use of the future information to help the procedure encounter an adversarial example at an earlier time $t$. At time $t-1$, the perturbable position set $P$ can be divided into two sets $P_1=\{p_{i}\}_{i=1}^{t-1}$ and $P_2=\{p_i\}_{i=t}^{m}$. $P_1$ is the set of positions that have been considered and $P_2$ is the set of positions to be considered in the future. Then we look ahead and pick the best position $p^*$ in $P_2$ to increase the chance of finding an adversarial example in the next time $t$. So the recurrence relation \ref{PDP} can be optimized as:
\begin{equation}\label{PDP1}
A(X, t):=\mathrm{TopK}(A(X,t-1)) \times S(X,p^*)
\end{equation}

This pseudo dynamic programming procedure is designed for GPU computing. It can make good use of the characteristics of parallel computing. For each step, the texts in $A(X, t)$ can be fed into classifier $F$ simultaneously as a batch to find adversarial examples and calculate evaluation scores.


\subsubsection{Score Functions}
Next, we explain how to realize $TopK(\cdot)$ for remembering \textbf{history} information and how to look ahead for the \textbf{future} in finding $p*$, which is the key to the PDP.

\paragraph{TopK($\cdot$)}
We use the score $I_s(X')$ to measure the importance of a text $X' \in A(X,t)$. It can be:
\begin{itemize}
	\item $I_s(X'):=1-F_{y*}(X')$ \ \hfill Untargeted attack
	\item $I_s(X'):=F_{\hat{y}}(X')$ \ \hfill Targeted ($\hat{y}$) attack
\end{itemize}
Operation $\mathrm{TopK}(\cdot)$ will preserve $K$ texts with highest score $I_s$. For an untargeted attack, it will preserve $K$ texts with the lowest confidence score for the gold label; For a targeted attack, it will preserve $K$ texts with the highest confidence score for the expected output label $\hat{y}$.

\paragraph{Looking Ahead} 
We call $A(X,t)$ as a \textbf{configuration} at time $t$. Let $X_{w_p\leftarrow w}$ denote the text after replacing the word $w_p$ in position $p$ of $X$ by $w$. The importance score of the perturbed position $p$ under the current configuration $A(X,t)$ is $I_p(A(X,t),p)$. It can be:
\begin{itemize}
	\item Untargeted attack: \\
	$I_p(A(X,t),p):=1-\min\limits_{X'\sim A,w \in S(X,p)}\{F_{y*}(X_{w_p \leftarrow w}')\}$
	\item  Targeted attack:\\
	$I_p(A(X,t),p):=\max\limits_{X'\sim A, w \in S(X,p)}\{F_{\hat{y}}(X_{w_p \leftarrow w}')\}$
\end{itemize}
where $X'\sim A$ means drawing some texts from $A(X,t)$ with probability proportional to $I_s(X')$. Then we have the position $p*$, which has the highest score $I_p$, for the next step $t$ to consider:
\begin{equation*}
p*:=\mathop{\arg\max}_{p \in P_2}\{I_p(A(X,t-1),p)\}
\end{equation*} 

Under the white-box setting, gradient information also can be used to measure the importance of position $p$.

The overall PDP algorithm is shown in Algorithm 1. It is a polynomial-time algorithm ($\mathcal{O}(n^2 \cdot poly(|F|,n))$ in the worst case, and the proof is in the supplementary material). $poly(|F|,n)$ represents prediction time of classifier $F$ for an input with length $n$. It is a polynomial function.

\subsection{Type-3 Problem: Robustness Verification}\label{sec:Robustness_Certification}
Verification is a method to prove the correctness of a system with respect to a certain property via formal methods of mathematics. If we can prove formula \ref{Eq:L0Robustness} is true for a certain radius $r$ (Type-3 problem), that means $r$ is a lower bound of maximum safe radius. Via combining the upper and lower bound, we can figure out the boundary of the safe regions. Generally speaking, proving is much more difficult than find a counter example (Type-1 problem), which needs to enumerate the exponential space or design a theorem proving algorithm. Several over-approximate verification methods like Interval Bound Propagation (IBP) \cite{JiaRGL19,HuangSWDYGDK19} have recently been introduced from image to NLP. Limited by time cost, scaling to large neural networks is a challenge for these methods. In this section, we introduce a property of $L_0$-robustness, which is helpful for certifying robustness when radius $r$ is fixed. It can also be used to improve the efficiency of other verification methods.

\begin{theorem}\label{theoremP}
	For any fixed $r$, Type-3 problem is in time complexity class $\mathcal{P}$.
\end{theorem}
\begin{proof}
	Suppose that a classifier $F$ can output a prediction for an input $X$ with length $n$ in $poly(|F|,n)$ time and $X$ has $m$ perturbable positions. For a given $r$, we have:
	\begin{equation*}
	|\Omega_r(X)| \leq \binom{m}{r}\cdot v^r\leq\binom{n}{r}\cdot v^r
	\end{equation*}
	where $v=\max_{p\in P}\{|S(X,p)|\}$. We know that the size of $\Omega_r(X)$ is bounded by $\mathcal{O}((nv)^r)$. So, one can test all the possible substitutions in $\Omega_r(X)$ in $\mathcal{O}((nv)^r) \cdot poly(|F|,n)$ time to answer problems of Type-3.
\end{proof}

Such conclusions are specific for NLP area owing to its discrete nature. In many cases, the upper bound of $r$ can be given by our PDP algorithm. In such a situation, we can directly enumerate all the possible substitutions to prove the absence of adversarial examples within $r$ (or formula \ref{Eq:L0Robustness} holds) in polynomial time. The enumeration procedure accomplished by a simple prover (SP), returns ``Certified Robustness'' or `` Found an adversarial example''.  After the absence of adversarial examples in $\Omega_r(X)$ is proved, $r$ is a lower bound for the maximum $L_0$-robustness radius.

All the possible substitutions compose a polynomial-time verifiable formal proof for the absence of adversarial examples. A checkable proof can make the result more convincing. If an algorithm finds an adversarial example, we can check the result easily. However, if an algorithm reports no adversarial examples, it is difficult to figure out whether there are indeed no adversarial samples or the verification algorithm has some bugs.

Under the white-box setting, the gradient information can be used to accelerate the verification algorithm. The basic idea is to test more sensitive positions first. Once an adversarial example occurs, the program can be terminated. Let $\left\|\partial F_{y^*}(X)/\partial w_p\right\|_1$ denote sensitivity score of perturbable position $p$, we can pre-sort the perturbable positions in $P$ based on the sensitivity score.

\subsection{Type-4 Problem: Robustness Metric}\label{RM}
Why are neural networks often fooled by small crafted perturbations, but have good generalization to noisy inputs in the real environment? How about the ability of a model to resist perturbation outside the robust radius? These questions promote us to analyze robustness from another perspective: the quantity of adversarial examples. Sometimes, it is difficult to enumerate all the adversarial examples in the perturbation space. 

We relax the universal quantifier ``$\forall$'' in formula \ref{Eq:Robustness} to a quantitative version as word-level robustness metric $PR$:
\begin{equation} \label{Eq:PR}
PR:=\frac{|\{X' : X' \in \Omega_r(X) \wedge F(X')=y^*\}|}{|\Omega_r(X)|},
\end{equation}
where we can see that $1$-$PR$ is the proportion of adversarial examples. Therefore, the higher the $PR$ value is, the less vulnerable the classifier $F$ is to be fooled by random perturbations around the point $X$. When $PR$=1, it is equivalent to formula \ref{Eq:Robustness}. 

Apparently, the exact computation of $PR$ is essentially a \textbf{Type-4 problem}. For a long input sequence, calculating the value of $PR$ is infeasible at the moment due to the limitation of computational power. As an alternative, we estimate $PR$ via a statistical method. Suppose that $X_1,X_2,...,X_N$ are taken from $\Omega_r(X)$ with uniform sampling, then an estimator $\hat{PR}$ for $PR$ is:
\begin{equation}
\hat{PR}:=\frac{1}{N} \sum \limits_{i=1}^{N}\mathbb{I}(F(X_i)=y^*)
\end{equation}

The satisfaction of $(F(X_i)=y^*)$ can be seen as Bernoulli random variable $Y_i$, i.e., $Y_i \sim Bernoulli(PR)$. So, if we want estimator $\hat{PR}$ to satisfy a prior guarantee such as the probability of producing an estimation which deviates from its real value $PR$ by a certain amount $\epsilon$ is less than $\delta$, the following must hold:
\begin{equation}\label{Hoeffding}
Pr(|\hat{PR}-PR|<\epsilon) > 1-\delta
\end{equation}
Based on Hoeffding’s inequality:
\begin{equation*}
Pr(| \frac{1}{N} \sum \limits_{i=1}^{N}Y_i-PR| \geq \epsilon) \leq 2e^{-2N\epsilon^2}
\end{equation*}
For given parameters $\epsilon$ and $\delta$, the estimator $\hat{PR}$ satisfies formula (\ref{Hoeffding}) if:
\begin{equation}
N>\frac{1}{2\epsilon^2}\ln{\frac{2}{\delta}}
\end{equation}

$\hat{PR}$ is a metric for a model's susceptibility to random perturbations with rigorous statistical guarantees. As the error bound and sample complexity is similar to those in PAC theory, we also call it PAC-style robustness metric.

\section{Experiments}
In this section, we design three sets of experiments to study the three problems and methods we proposed. 
\subsection{General Experiment Setup}
\paragraph{Tasks}
We conduct experiments on two important NLP tasks: text classification and textual entailment. MR \cite{mr} and IMDB \cite{imdb} are sentence-level and document-level sentiment classification respectively on positive and negative movie reviews. SNLI \cite{snli} is used to learn to judge the relationship between two sentences: whether the second sentence can be derived from entailment, contradiction, or neutral relationship with the first sentence.
\paragraph{Target Models}
For each task, we choose two widely used models,  bidirectional LSTM (BiLSTM) \cite{bilstm} and BERT \cite{bert} as the attacking target models. For BiLSTM, we used a 1-layer bidirectional LSTM with 150 hidden units, and 300-dimensional pre-trained GloVe \cite{glove} word embeddings. We used the 12-layer based version of BERT model with 768 hidden units and 12 heads, with 110M parameters. Details of the data and the classification accuracy on the test set of the models are listed in Table \ref{datasets}.

\begin{table}[h]
	\centering
	\resizebox{0.5\textwidth}{!}{
		\begin{small}
			\begin{tabular}{cccc|cc} 
				\toprule
				Dataset&Avg Len&Train&Test&BiLSTM&BERT\\
				\hline 
				MR&20(2/50)&9K&1K&82.47&89.60\\
				IMDB&215(6/2K)&25K&25K&91.23&92.27\\
				\hline
				SNLI&8(2/30)&570K&10K&84.43&90.50\\
				\bottomrule
			\end{tabular}
		\end{small}
	}
	\caption{Overview of the datasets and the test accuracy of target models. The numbers in brackets of column ``Avg Len'' are the minimum/maximum length. }
	\label{datasets}
\end{table}	

\subsection{Type-2 Problem: Attack Evaluation}
\paragraph{Baselines} We use two state-of-the-art adversarial crafting methods (TextFooler \cite{textfooler} and SemPSO \cite{seme}) as references to compare the search capability of PDP. TextFooler is a greedy algorithm and SemPSO is a particle-swarm-based algorithm. They all focus on Type 1 problem while PDP focuses on Type 2 problem.

\begin{table*}
	\resizebox{1\textwidth}{!}{
		\centering
		\begin{small}
			\begin{tabular}{lcc||ccc|ccc|ccc}
				\toprule
				\multirow{2}*{Dataset} & \multirow{2}*{Model} & \multirow{2}*{\#Attacks} & \multicolumn{3}{c|}{SemPSO}&\multicolumn{3}{c|}{TextFooler}&\multicolumn{3}{c}{PDP}\\
				&&&\#Succ&\#Win&\%S&\#Succ&\#Win&\%S&\#Succ&\#Win&\%S \\
				\hline
				\multirow{2}*{MR}&BiLSTM&880&636(72.27\%)&0&10.64&484(55.00\%)&0&12.09&\textbf{655(74.43\%)}&\textbf{33}&\textbf{10.44}\\
				&BERT&956&580(60.67\%)&0&12.10&323(33.79\%)&0&13.96&\textbf{621(64.96\%)}&\textbf{30}&\textbf{11.80}\\
				\hline
				\multirow{2}*{IMDB}&BiLSTM&1000&947(94.7\%)&0&4.58&854(85.4\%)&0&6.78&\textbf{989(98.9\%)}&\textbf{599}&\textbf{3.11}\\
				&BERT&1000&871(87.1\%)&0&4.31&714(71.4\%)&0&8.47&\textbf{899(89.9\%)}&\textbf{498}&\textbf{2.87}\\
				\hline
				\multirow{2}*{SNLI}&BiLSTM&1000&505(50.5\%)&0&15.99&592(59.2\%)&0&15.76&\textbf{764(76.4\%)}&\textbf{31}&\textbf{14.91}\\
				&BERT&1000&587(58.7\%)&0&16.10&636(63.6\%)&0&15.83&\textbf{845(84.5\%)}&\textbf{30}&\textbf{15.09}\\
				\bottomrule
			\end{tabular}
		\end{small}
	}
	\caption{The attack results of different methods. \#Attacks is the number of texts to be attacked. \#Succ is the number of successful attacks.  \#Win is the number of successful attacks crafted with the least substitutions for the same texts among various attack methods. \%S is the avarage percentage of substitutited words.}
	\label{attack_reslults}
\end{table*}

\paragraph{Metrics} We evaluate the performance of these attack methods including the rate of successful attacks and the percentage of word substitution. A smaller percentage (or number) of word substitution means a tighter upper bound for the maximum $L_0$-robustness radius.

\paragraph{Settings}
For a fair comparison, we set the same candidate set and constraints for different attack methods. The candidate is generated by HowNet \cite{hownet} and similarities of word embeddings. HowNet is arranged by the sememe and can find the potential semantic-preserving words. Word embeddings can further help to select the most similar candidate words. So, we generate $S(X,p)$ via cleaning the synonyms obtained by HowNet with cosine similarity of word embeddings. We reserve top $\eta$ ($\eta=5$) synonyms as candidates for each position. 

For MR, we experiment on all the test texts classified correctly. For IMDB and SNLI, we randomly sample 1000 texts classified correctly from the test set. Following \cite{AlzantotSEHSC18,seme}, only the hypotheses are perturbed for SNLI. The adversarial examples with modification rates less than 25\% are considered valid.

\paragraph{Attack Results} 	We present the average percentage of substitutions (\%S) in Table \ref{attack_reslults} and the number of times each method ``wins'' the others in terms of substitution length (\#Win). The experimental results show that PDP always gives adversarial examples with fewer substitutions. Especially for the long-text dataset, IMDB, 599 (59.9\%) adversarial examples found by PDP contains the least word substitutions for BiLSTM (the remaining 40.1\% holds the same number of substitutions with others). The examples crafted by PDP contain very few substitutions, such as average \textbf{4.52} word substitutions for BERT on IMDB whose average number of words is 215. The comparison of the substitution length on IMDB is shown in Figure \ref{num_changes}.  Besides, PDP achieves the highest attack success rates on all three datasets and two target models. These experimental results indicate PDP has stronger search capabilities. Then, we repurpose PDP attack to evaluate the robustness, and Figure \ref{upper_bound} shows that PDP can provide a tighter bound for the maximum robustness radii compared with other attacking methods. More experimental results are shown in the supplementary.


\begin{figure}
	\centering
	\subfigure[IMDB-BiLSTM]{
		\includegraphics[width=0.22\textwidth,height=0.13\textheight]{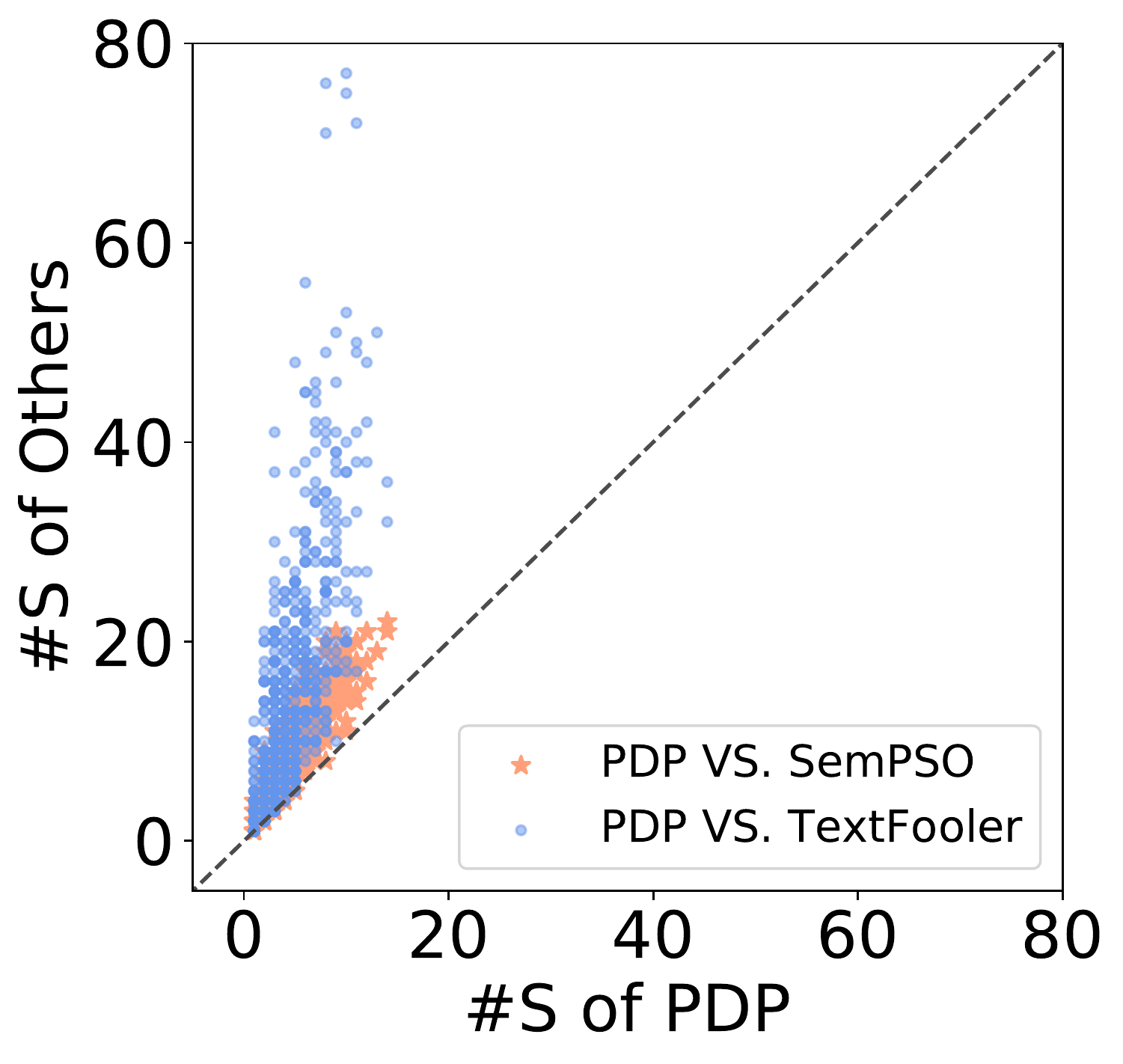}
	}
	\subfigure[IMDB-BERT]{
		\label{imdb-bert}
		\includegraphics[width=0.22\textwidth,height=0.13\textheight]{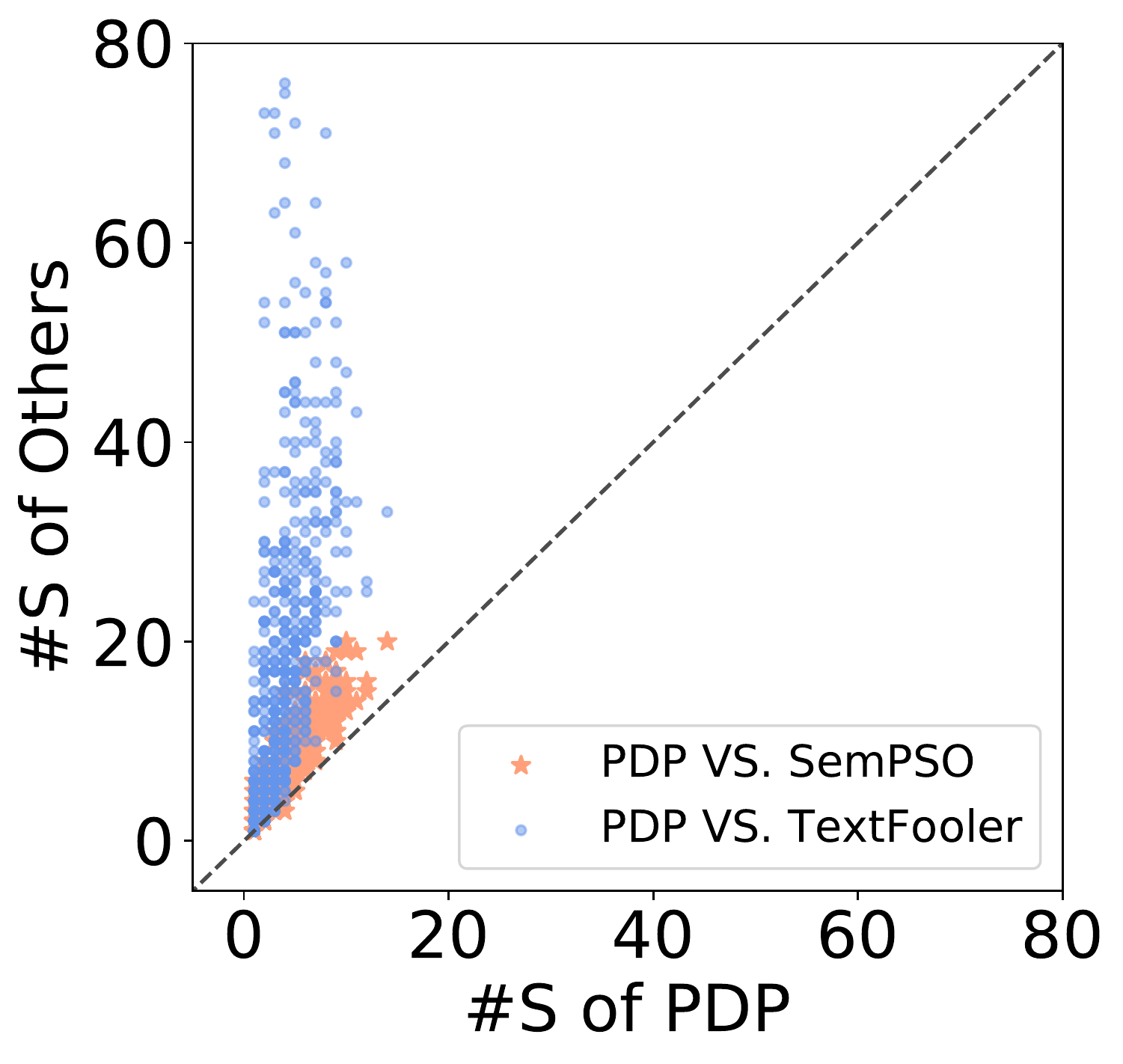}
	}
	\caption{Comparision of the number of substituted words of different methods on IMDB. Each point represents a text ($x$-axis is the number of substituted (\#S) words of PDP and $y$-axis is that of other attack methods). Points over the diagonal are where PDP finds an adversarial example with fewer substitutions.}
	\label{num_changes}	
\end{figure}





\subsection{Type-3 Problem: Robustness verification} \label{Rcert}
For a given $L_0$ distance $r$ ($r$=$1$\ to\ $4$), the certified results on 200 randomly sampled test instances are shown in Table \ref{certi_robust}. We can have the three findings below. (1) The percentage of certified robustness is decreasing with the increase of radius $r$. (2) For many short-text tasks (MR and SNLI), considering $r\leq4$ is sufficient because most regions cannot resist 4-word substitutions. For example, only 6.42\% regions of BERT can resist any 4-word substitutions adversarial attack on SNLI. (3) For the long-text task IMDB, BERT has more regions (61.52\%) that can resist any 3-word substitutions attack compared with BiLSTM. It takes a long time to certify robustness when $r$=$4$, so we don't show the results. Experimental results also show that this simple verification method is effective for many NLP tasks.

\begin{table*}
	\resizebox{1\textwidth}{!}{
		\begin{tabular}{lc|rrrr|rrrr|rrrr|rrrrr}
			\toprule
			\multirow{3}*{Dataset} & \multirow{3}*{Model}& \multicolumn{4}{c|}{r=1} &\multicolumn{4}{c|}{r=2}&\multicolumn{4}{c|}{r=3}&\multicolumn{4}{c}{r=4}\\
			\cline{3-18}
			&&\multicolumn{2}{c}{Found}& \multicolumn{2}{c|}{Certified}&\multicolumn{2}{c}{Found}& \multicolumn{2}{c|}{Certified}&\multicolumn{2}{c}{Found}& \multicolumn{2}{c|}{Certified}&\multicolumn{2}{c}{Found}& \multicolumn{2}{c}{Certified}\\
			&&\%F&T(s)&\%C&T(s)&\%F&T(s)&\%C&T(s)&\%F&T(s)&\%C&T(s)&\%F&T(s)&\%C&T(s)\\
			\hline
			\multirow{2}*{MR}&BiLSTM&36.00&0.01&64.00&0.02&58.00&0.04&42.00&0.05&72.00&0.30&28.00&0.27&78.00&3.62&22.00&2.26\\
			&BERT&20.00&0.12&80.00&0.24&40.00&1.95&60.00&2.49&56.50&28.86&43.50&20.95&67.50&256.78&32.50&46.23\\
			\hline
			\multirow{2}*{IMDB}&BiLSTM&15.59&0.04&84.41&0.13&31.99&1.93&68.01&2.89&45.50&2.47&54.50&953.19&-&-&-&-\\
			&BERT&12.66&2.89&87.34&3.11&25.91&3.40&74.09&246.54&38.48&7.97&61.52&6448.39&-&-&-&-\\
			\hline
			\multirow{2}*{SNLI}&BiLSTM&56.90&0.04&43.10&0.01&76.87&0.03&23.13&0.07&82.52&0.14&17.48&0.11&84.63&0.66&15.37&0.23\\
			&BERT&71.43&0.03&28.57&0.01&88.01&0.07&11.99&0.06&92.37&0.47&7.63&0.07&93.58&3.17&6.42&0.04\\
			\bottomrule
		\end{tabular}
	}
	\caption{Certified robustness. ``Found''  and ``Certified'' are the abbreviations for ``an adversarial example found'' and ``certified to be robust'' respectively. \%F and \%C  are the percentage of ``Found''  and ``Certified''. ``T'' is the average time. }
	\label{certi_robust}
\end{table*}

\begin{figure}
	\subfigure[IMDB-BiLSTM]{
		\includegraphics[width=3.6cm,height=0.12\textheight]{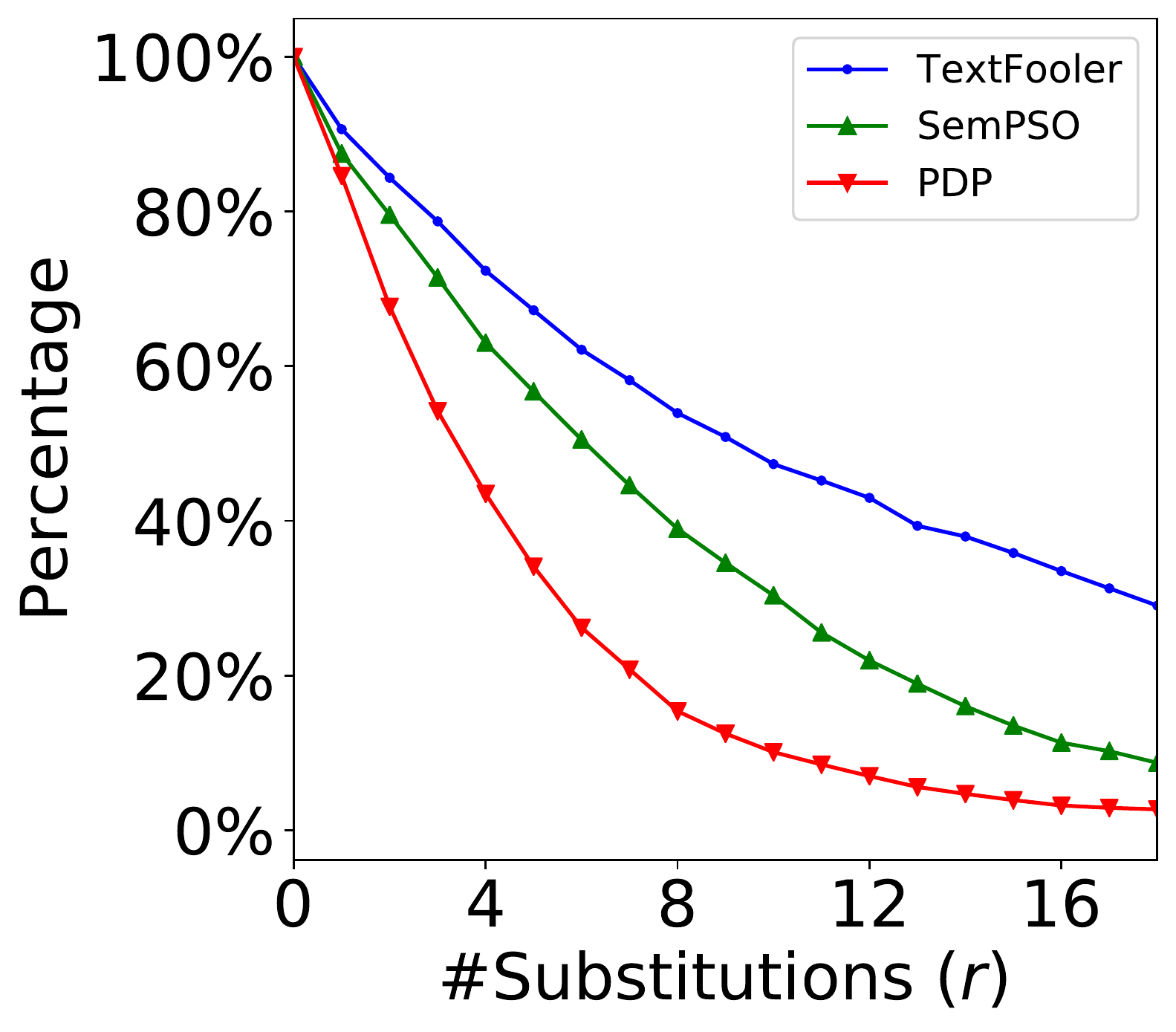}
	}
	\subfigure[IMDB-BERT]{
		\includegraphics[width=3.6cm,height=0.12\textheight]{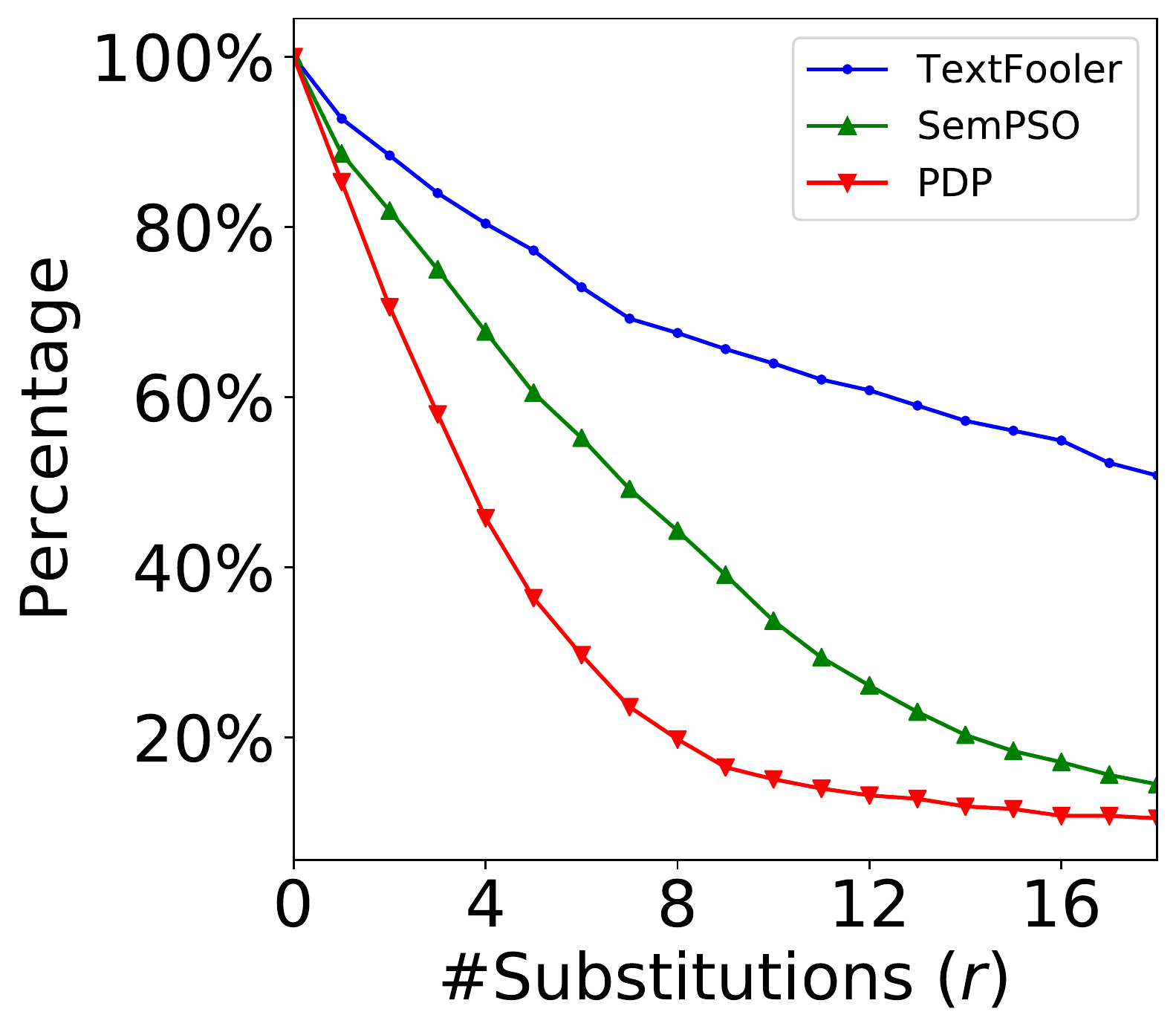}
	}
	\caption{The percentage of regions ($\Omega_r(X)$) that do not yield to different attacking methods in each perturbation radius. $x$-axis can be seen as the upper bounds given by different attacking methods.}
	\label{upper_bound}
\end{figure}

\subsection{Type-4 Problem: Robustness Metric}
\begin{figure*}[h]
	\centering
	\includegraphics[width=0.75\textheight]{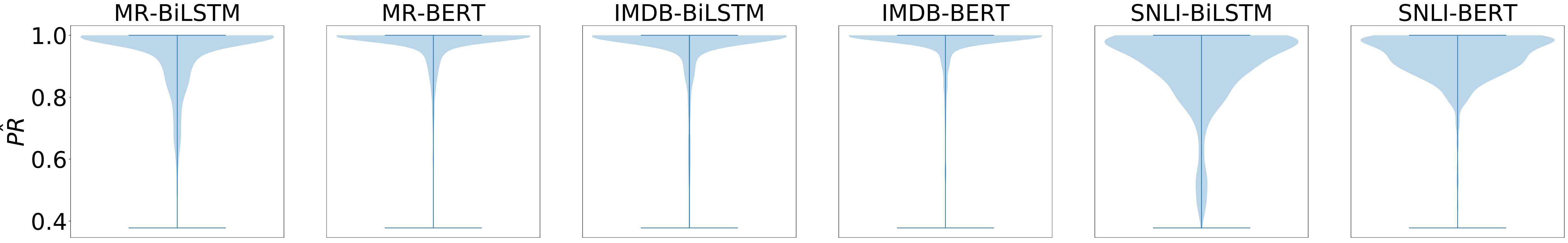}
	\caption{Violin plots of robustness score on the test set when $r$ is 25\% of the length of sentence. The width of $x$-axis represents the frequency of corresponding $\hat{PR}$ in $y$-axis. Top and bottom horizontal lines are the maximum and minimum value of $\hat{PR}$.}
	\label{robus_eval}
\end{figure*} 

We evaluate the robustness score (Equation \ref{Eq:PR}) of different models on different tasks. The evaluation is performed on the randomly sampled 1000 test data and the sample size $N$ is 5000 ($\epsilon$=$0.025$, $\delta$=$0.005$).
The violin plots of $\hat{PR}$ are shown in Figure \ref{robus_eval}. As most attacking algorithms limit the maximum perturbation ratio to smaller than $25\%$, we set r to 25\% of the length of the sentence.

Most of the shadows in all sub-figures are close to the top horizontal line (maximum $\hat{PR}$), which means that most regions have high robustness scores $\hat{PR}$. Take BERT model on IMDB task as an example, $89.9\%$ regions are found with adversarial examples as shown in Table \ref{attack_reslults}, which indicates the ``vulnerability'' of the model. However, via robustness metric, we find that $90.66\%$ regions achieve $\hat{PR}$ larger than $0.9$. It means, most regions ($90.66\%$) can resist random word perturbations with high probability ($>0.9$). A conclusion can be drawn: these well-trained models are usually robust to word substitutions in a non-adversarial environment.

For a well-trained model, the adversarial examples crafted by word substitution are almost everywhere and close to the normal point in the perturbation space, but their proportion is very low. In 2019, Stutz et al. pointed out that on-manifold robustness is essentially generalization and on-manifold adversarial examples are generalization errors \cite{Stutz0S19}. Suppose users' selection from a synonym candidate is similar to the process of rolling a die, which means the usage of a word for position $p$ in $S(X,p)$ is conditional on a latent variable $z_p$, i.e. $Pr(w_p|z_p)$ corresponding to the underlying, low dimensional manifold. All possible substitutions in $\Omega(X)$ can be seen as on the manifold corresponding to a latent variable $Z=(z_{p_1},z_{p_2},..., z_{p_k})$. Thus, the adversarial examples found by attacking algorithms are essentially on-manifold generalization errors. From this perspective, we can explain why a well-trained model like BERT has a good generalization but can be easily attacked by word substitutions. 

For all three tasks, BERT always presents better robustness performance. For instance, on MR task, the proportion of regions with $\hat{PR}$ larger than $0.9$ is $79.22\%$ and $90.97\%$ for BiLSTM and BERT respectively. It means BERT is always more robust outside the safe regions. Our robustness metric presents its superiority to the traditional model metrics including accuracy.

\section{Related Work}
Existing works about the word-level robustness problem mainly focus on three lines of research points. 

\paragraph{Adversarial Examples} Various attack algorithms are developed for generating adversarial examples via substitutions including gradient descent methods \cite{SatoSS018,LSBLS18,Wang0D021}, genetic algorithm \cite{AlzantotSEHSC18}, particle-swarm-based method \cite{seme}, greedy-based methods \cite{Ren19,textfooler} and BERT-based methods \cite{LiMGXQ20,GargR20}. They focus on how to generate adversarial examples effectively and simply regard robustness as the opposite of attack success rate. 

\paragraph{Robustness Verfication} \cite{JiaRGL19,Huang19,ShiZCHH20} migrate the over-approximate method IBP from the image field to certify the robustness in the continuous space based on word embedding. Although they can give a provably robust to all possible perturbations within the constraints, the limitation is that a model which is not robust in continuous space can be robust in discrete space, as the vectors that can fool the model may not correspond to any real words. \cite{YeGL20} introduce a randomized smoothing-based method to certify the robustness of a smoothed classifier. 
Existing robustness evaluation works focus on robustness verification which aims to verify the absolute safe for a given model in the whole perturbation space. They ignore the safe sub-regions and unsafe regions. 

\paragraph{Defense} Naturally, the final goal is to defend against attacking and improve the robustness of models. Adversarial data augmentation (ADA) is one of the most effective empirical methods. \cite{Ren19,textfooler,LiMGXQ20,GargR20,Wang0D021} adopt the adversarial examples generated by their attack methods for adversarial training and achieve some robustness improvement. Adversarial training is another similar method, which incorporates a min-max optimization between adversarial perturbations and the models by adding norm-bounded perturbations to word embeddings \cite{MadryMSTV18,ZhuCGSGL20}. They depend on search algorithms for adversarial examples, so our PDP with better search ability can provide support for these robustness enhancement methods.

\section{Conclusion}
Overall, we build a formal framework to study the word-level robustness of the deep-learning-based NLP systems. We repurpose the attack method for robustness evaluation and design a pseudo-dynamic programming framework for crafting adversarial examples with fewer substitutions to provide a tighter upper bound. Besides, we notice that the absence of adversarial examples within any fixed radius can be verified in polynomial time, and give a simple prover to certify the lower bound. Experimental results show that our methods can provide tighter bounds for robustness evaluation, and most state-of-the-art models like BERT cannot resist a few word substitutions. Further, we discuss the robustness from the view of quantification and introduce a PAC-style metric to show they are robust to random perturbations, as well as explain why they generalize well but are poor in resisting adversarial attacks. It can be helpful to studying defense and interpretability of NLP models. 

\bibliography{aaai22}

\begin{thebibliography}{29}
\providecommand{\natexlab}[1]{#1}

\bibitem[{Alzantot et~al.(2018)Alzantot, Sharma, Elgohary, Ho, Srivastava, and
  Chang}]{AlzantotSEHSC18}
Alzantot, M.; Sharma, Y.; Elgohary, A.; Ho, B.; Srivastava, M.~B.; and Chang,
  K. 2018.
\newblock Generating Natural Language Adversarial Examples.
\newblock In \emph{Proceedings of the 2018 Conference on Empirical Methods in
  Natural Language Processing, Brussels, Belgium, October 31 - November 4,
  2018}, 2890--2896.

\bibitem[{Baluta et~al.(2019)Baluta, Shen, Shinde, Meel, and
  Saxena}]{BalutaSSMS19}
Baluta, T.; Shen, S.; Shinde, S.; Meel, K.~S.; and Saxena, P. 2019.
\newblock Quantitative Verification of Neural Networks and Its Security
  Applications.
\newblock In \emph{Proceedings of the 2019 {ACM} {SIGSAC} Conference on
  Computer and Communications Security, {CCS} 2019, London, UK, November 11-15,
  2019}, 1249--1264.

\bibitem[{Belinkov and Bisk(2018)}]{BelinkovB18}
Belinkov, Y.; and Bisk, Y. 2018.
\newblock Synthetic and Natural Noise Both Break Neural Machine Translation.
\newblock In \emph{6th International Conference on Learning Representations,
  {ICLR} 2018, Vancouver, BC, Canada, April 30 - May 3, 2018, Conference Track
  Proceedings}.

\bibitem[{Bowman et~al.(2015)Bowman, Angeli, Potts, and Manning}]{snli}
Bowman, S.~R.; Angeli, G.; Potts, C.; and Manning, C.~D. 2015.
\newblock A large annotated corpus for learning natural language inference.
\newblock In \emph{Proceedings of the 2015 Conference on Empirical Methods in
  Natural Language Processing, {EMNLP} 2015, Lisbon, Portugal, September 17-21,
  2015}, 632--642.

\bibitem[{Conneau et~al.(2017)Conneau, Kiela, Schwenk, Barrault, and
  Bordes}]{bilstm}
Conneau, A.; Kiela, D.; Schwenk, H.; Barrault, L.; and Bordes, A. 2017.
\newblock Supervised Learning of Universal Sentence Representations from
  Natural Language Inference Data.
\newblock In \emph{Proceedings of the 2017 Conference on Empirical Methods in
  Natural Language Processing, {EMNLP} 2017, Copenhagen, Denmark, September
  9-11, 2017}, 670--680.

\bibitem[{Devlin et~al.(2019)Devlin, Chang, Lee, and Toutanova}]{bert}
Devlin, J.; Chang, M.; Lee, K.; and Toutanova, K. 2019.
\newblock {BERT:} Pre-training of Deep Bidirectional Transformers for Language
  Understanding.
\newblock In \emph{Proceedings of the 2019 Conference of the North American
  Chapter of the Association for Computational Linguistics: Human Language
  Technologies, {NAACL-HLT} 2019, Minneapolis, MN, USA, June 2-7, 2019, Volume
  1 (Long and Short Papers)}, 4171--4186.

\bibitem[{Dong and Dong(2006)}]{hownet}
Dong, Z.; and Dong, Q. 2006.
\newblock HowNet and the Computation of Meaning. {W}orld {S}cientific.

\bibitem[{Garg and Ramakrishnan(2020)}]{GargR20}
Garg, S.; and Ramakrishnan, G. 2020.
\newblock {BAE:} BERT-based Adversarial Examples for Text Classification.
\newblock In \emph{Proceedings of the 2020 Conference on Empirical Methods in
  Natural Language Processing, {EMNLP} 2020, Online, November 16-20, 2020},
  6174--6181.

\bibitem[{Huang et~al.(2019{\natexlab{a}})Huang, Stanforth, Welbl, Dyer,
  Yogatama, Gowal, Dvijotham, and Kohli}]{Huang19}
Huang, P.; Stanforth, R.; Welbl, J.; Dyer, C.; Yogatama, D.; Gowal, S.;
  Dvijotham, K.; and Kohli, P. 2019{\natexlab{a}}.
\newblock Achieving Verified Robustness to Symbol Substitutions via Interval
  Bound Propagation.
\newblock In \emph{Proceedings of the 2019 Conference on Empirical Methods in
  Natural Language Processing and the 9th International Joint Conference on
  Natural Language Processing, {EMNLP-IJCNLP} 2019, Hong Kong, China, November
  3-7, 2019}, 4081--4091.

\bibitem[{Huang et~al.(2019{\natexlab{b}})Huang, Stanforth, Welbl, Dyer,
  Yogatama, Gowal, Dvijotham, and Kohli}]{HuangSWDYGDK19}
Huang, P.; Stanforth, R.; Welbl, J.; Dyer, C.; Yogatama, D.; Gowal, S.;
  Dvijotham, K.; and Kohli, P. 2019{\natexlab{b}}.
\newblock Achieving Verified Robustness to Symbol Substitutions via Interval
  Bound Propagation.
\newblock In Inui, K.; Jiang, J.; Ng, V.; and Wan, X., eds., \emph{Proceedings
  of the 2019 Conference on Empirical Methods in Natural Language Processing
  and the 9th International Joint Conference on Natural Language Processing,
  {EMNLP-IJCNLP} 2019, Hong Kong, China, November 3-7, 2019}, 4081--4091.
  Association for Computational Linguistics.

\bibitem[{Jia et~al.(2019)Jia, Raghunathan, G{\"{o}}ksel, and Liang}]{JiaRGL19}
Jia, R.; Raghunathan, A.; G{\"{o}}ksel, K.; and Liang, P. 2019.
\newblock Certified Robustness to Adversarial Word Substitutions.
\newblock In \emph{Proceedings of the 2019 Conference on Empirical Methods in
  Natural Language Processing and the 9th International Joint Conference on
  Natural Language Processing, {EMNLP-IJCNLP} 2019, Hong Kong, China, November
  3-7, 2019}, 4127--4140.

\bibitem[{Jin et~al.(2020)Jin, Jin, Zhou, and Szolovits}]{textfooler}
Jin, D.; Jin, Z.; Zhou, J.~T.; and Szolovits, P. 2020.
\newblock Is {BERT} Really Robust? {A} Strong Baseline for Natural Language
  Attack on Text Classification and Entailment.
\newblock In \emph{The Thirty-Fourth {AAAI} Conference on Artificial
  Intelligence, {AAAI} 2020, New York, NY, USA, February 7-12, 2020},
  8018--8025.

\bibitem[{Katz et~al.(2017)Katz, Barrett, Dill, Julian, and
  Kochenderfer}]{KatzBDJK17}
Katz, G.; Barrett, C.~W.; Dill, D.~L.; Julian, K.; and Kochenderfer, M.~J.
  2017.
\newblock Reluplex: An Efficient {SMT} Solver for Verifying Deep Neural
  Networks.
\newblock In \emph{Computer Aided Verification - 29th International Conference,
  {CAV} 2017, Heidelberg, Germany, July 24-28, 2017, Proceedings, Part {I}},
  97--117.

\bibitem[{Li et~al.(2020)Li, Ma, Guo, Xue, and Qiu}]{LiMGXQ20}
Li, L.; Ma, R.; Guo, Q.; Xue, X.; and Qiu, X. 2020.
\newblock {BERT-ATTACK:} Adversarial Attack Against {BERT} Using {BERT}.
\newblock In \emph{Proceedings of the 2020 Conference on Empirical Methods in
  Natural Language Processing, {EMNLP} 2020, Online, November 16-20, 2020},
  6193--6202.

\bibitem[{Liang et~al.(2018)Liang, Li, Su, Bian, Li, and Shi}]{LSBLS18}
Liang, B.; Li, H.; Su, M.; Bian, P.; Li, X.; and Shi, W. 2018.
\newblock Deep Text Classification Can be Fooled.
\newblock In \emph{Proceedings of the Twenty-Seventh International Joint
  Conference on Artificial Intelligence, {IJCAI} 2018, July 13-19, 2018,
  Stockholm, Sweden}, 4208--4215.

\bibitem[{Maas et~al.(2011)Maas, Daly, Pham, Huang, Ng, and Potts}]{imdb}
Maas, A.~L.; Daly, R.~E.; Pham, P.~T.; Huang, D.; Ng, A.~Y.; and Potts, C.
  2011.
\newblock Learning Word Vectors for Sentiment Analysis.
\newblock In \emph{The 49th Annual Meeting of the Association for Computational
  Linguistics: Human Language Technologies, Proceedings of the Conference,
  19-24 June, 2011, Portland, Oregon, {USA}}, 142--150.

\bibitem[{Madry et~al.(2018)Madry, Makelov, Schmidt, Tsipras, and
  Vladu}]{MadryMSTV18}
Madry, A.; Makelov, A.; Schmidt, L.; Tsipras, D.; and Vladu, A. 2018.
\newblock Towards Deep Learning Models Resistant to Adversarial Attacks.
\newblock In \emph{6th International Conference on Learning Representations,
  {ICLR} 2018, Vancouver, BC, Canada, April 30 - May 3, 2018, Conference Track
  Proceedings}.

\bibitem[{Neekhara et~al.(2019)Neekhara, Hussain, Dubnov, and
  Koushanfar}]{NeekharaHDK19}
Neekhara, P.; Hussain, S.; Dubnov, S.; and Koushanfar, F. 2019.
\newblock Adversarial Reprogramming of Text Classification Neural Networks.
\newblock In \emph{Proceedings of the 2019 Conference on Empirical Methods in
  Natural Language Processing and the 9th International Joint Conference on
  Natural Language Processing, {EMNLP-IJCNLP} 2019, Hong Kong, China, November
  3-7, 2019}, 5215--5224.

\bibitem[{Pang and Lee(2005)}]{mr}
Pang, B.; and Lee, L. 2005.
\newblock Seeing Stars: Exploiting Class Relationships for Sentiment
  Categorization with Respect to Rating Scales.
\newblock In \emph{{ACL} 2005, 43rd Annual Meeting of the Association for
  Computational Linguistics, Proceedings of the Conference, 25-30 June 2005,
  University of Michigan, {USA}}, 115--124.

\bibitem[{Pennington, Socher, and Manning(2014)}]{glove}
Pennington, J.; Socher, R.; and Manning, C.~D. 2014.
\newblock Glove: Global Vectors for Word Representation.
\newblock In \emph{Proceedings of the 2014 Conference on Empirical Methods in
  Natural Language Processing, {EMNLP} 2014, October 25-29, 2014, Doha, Qatar,
  {A} meeting of SIGDAT, a Special Interest Group of the {ACL}}, 1532--1543.

\bibitem[{Ren et~al.(2019)Ren, Deng, He, and Che}]{Ren19}
Ren, S.; Deng, Y.; He, K.; and Che, W. 2019.
\newblock Generating Natural Language Adversarial Examples through Probability
  Weighted Word Saliency.
\newblock In \emph{Proceedings of the 57th Conference of the Association for
  Computational Linguistics, {ACL} 2019, Florence, Italy, July 28- August 2,
  2019, Volume 1: Long Papers}, 1085--1097.

\bibitem[{Ribeiro, Singh, and Guestrin(2018)}]{SinghGR18}
Ribeiro, M.~T.; Singh, S.; and Guestrin, C. 2018.
\newblock Semantically Equivalent Adversarial Rules for Debugging {NLP} models.
\newblock In \emph{Proceedings of the 56th Annual Meeting of the Association
  for Computational Linguistics, {ACL} 2018, Melbourne, Australia, July 15-20,
  2018, Volume 1: Long Papers}, 856--865.

\bibitem[{Sato et~al.(2018)Sato, Suzuki, Shindo, and Matsumoto}]{SatoSS018}
Sato, M.; Suzuki, J.; Shindo, H.; and Matsumoto, Y. 2018.
\newblock Interpretable Adversarial Perturbation in Input Embedding Space for
  Text.
\newblock In \emph{Proceedings of the Twenty-Seventh International Joint
  Conference on Artificial Intelligence, {IJCAI} 2018, July 13-19, 2018,
  Stockholm, Sweden}, 4323--4330.

\bibitem[{Shi et~al.(2020)Shi, Zhang, Chang, Huang, and Hsieh}]{ShiZCHH20}
Shi, Z.; Zhang, H.; Chang, K.; Huang, M.; and Hsieh, C. 2020.
\newblock Robustness Verification for Transformers.
\newblock In \emph{8th International Conference on Learning Representations,
  {ICLR} 2020, Addis Ababa, Ethiopia, April 26-30, 2020}.

\bibitem[{Stutz, Hein, and Schiele(2019)}]{Stutz0S19}
Stutz, D.; Hein, M.; and Schiele, B. 2019.
\newblock Disentangling Adversarial Robustness and Generalization.
\newblock In \emph{{IEEE} Conference on Computer Vision and Pattern
  Recognition, {CVPR} 2019, Long Beach, CA, USA, June 16-20, 2019}, 6976--6987.

\bibitem[{Wang et~al.(2021)Wang, Yang, Deng, and He}]{Wang0D021}
Wang, X.; Yang, Y.; Deng, Y.; and He, K. 2021.
\newblock Adversarial Training with Fast Gradient Projection Method against
  Synonym Substitution Based Text Attacks.
\newblock In \emph{Thirty-Fifth {AAAI} Conference on Artificial Intelligence,
  {AAAI} 2021, Thirty-Third Conference on Innovative Applications of Artificial
  Intelligence, {IAAI} 2021, The Eleventh Symposium on Educational Advances in
  Artificial Intelligence, {EAAI} 2021, Virtual Event, February 2-9, 2021},
  13997--14005. {AAAI} Press.

\bibitem[{Ye, Gong, and Liu(2020)}]{YeGL20}
Ye, M.; Gong, C.; and Liu, Q. 2020.
\newblock {SAFER:} {A} Structure-free Approach for Certified Robustness to
  Adversarial Word Substitutions.
\newblock In \emph{Proceedings of the 58th Annual Meeting of the Association
  for Computational Linguistics, {ACL} 2020, Online, July 5-10, 2020},
  3465--3475.

\bibitem[{Zang et~al.(2020)Zang, Qi, Yang, Liu, Zhang, Liu, and Sun}]{seme}
Zang, Y.; Qi, F.; Yang, C.; Liu, Z.; Zhang, M.; Liu, Q.; and Sun, M. 2020.
\newblock Word-level Textual Adversarial Attacking as Combinatorial
  Optimization.
\newblock In \emph{Proceedings of the 58th Annual Meeting of the Association
  for Computational Linguistics, {ACL} 2020, Online, July 5-10, 2020},
  6066--6080.

\bibitem[{Zhu et~al.(2020)Zhu, Cheng, Gan, Sun, Goldstein, and
  Liu}]{ZhuCGSGL20}
Zhu, C.; Cheng, Y.; Gan, Z.; Sun, S.; Goldstein, T.; and Liu, J. 2020.
\newblock FreeLB: Enhanced Adversarial Training for Natural Language
  Understanding.
\newblock In \emph{8th International Conference on Learning Representations,
  {ICLR} 2020, Addis Ababa, Ethiopia, April 26-30, 2020}.

\end{thebibliography}
\end{document}